%% file: paper.tex
\documentclass{article}
\usepackage{times}
\usepackage{graphicx} 
\usepackage{subfigure}
\usepackage{natbib}
\usepackage{algorithm}
\usepackage{algorithmic}
\usepackage{hyperref}
\usepackage{paralist}
\usepackage{amssymb}
\usepackage{amsmath}
\usepackage{amsthm}
\usepackage{enumitem}
\usepackage{microtype}
\usepackage[textsize=tiny]{todonotes}

\usepackage[accepted]{icml2016}

\graphicspath{{figures/}}

\allowdisplaybreaks

\input{def}

\begin{document}
\twocolumn[
\icmltitle{Doubly Robust Off-policy Value Evaluation for Reinforcement Learning}
\icmlauthor{Nan Jiang}{nanjiang@umich.edu}
\icmladdress{Computer Science \& Engineering, University of Michigan}
\icmlauthor{Lihong Li}{lihongli@microsoft.com}
\icmladdress{Microsoft Research}

\icmlkeywords{reinforcement learning, off-policy evaluation, importance sampling}

\vskip 0.3in
]

\begin{abstract}
We study the problem of \emph{off-policy value evaluation} in reinforcement learning (RL), where one aims to estimate the value of a new policy based on data collected by a different policy. This problem is often a critical step when applying RL to real-world problems. Despite its importance, existing general methods either have uncontrolled bias or suffer high variance. In this work, we extend the \emph{doubly robust} estimator for bandits to sequential decision-making problems, which gets the best of both worlds: it is guaranteed to be unbiased and can have a much lower variance than the popular importance sampling estimators. We demonstrate the estimator's accuracy in several benchmark problems, and illustrate its use as a subroutine in safe policy improvement. We also provide theoretical results on the inherent hardness of the problem, and show that our estimator can match the lower bound in certain scenarios.
\end{abstract}

\section{Introduction}

We study the \emph{off-policy value evaluation} problem, where one aims to estimate the value of a policy with data collected by another policy~\cite{Sutton98Reinforcement}.  This problem is critical in many real-world applications of reinforcement learning (RL), whenever it is infeasible to estimate policy value by running the policy because doing so is expensive, risky, or unethical/illegal. In robotics and business/marketing applications, for instance, it is often risky (thus expensive) to run a policy without an estimate of the policy's quality~\cite{Li11Unbiased,Bottou13Counterfactual,thomas2015high}.  In medical and public-policy domains~\cite{Murphy01Marginal,Hirano03Efficient}, it is often hard to run a controlled experiment to estimate the treatment effect, and off-policy value evaluation is a form of counterfactual reasoning that infers the causal effect of a new intervention from historical data~\cite{Holland86Statistics,Pearl09Causality}. 


There are roughly two classes of approaches to off-policy value evaluation.  The first is to fit an MDP model from data via regression, and evaluate the policy against the model. 
Such a regression based approach has a relatively low variance and works well when the model can be learned to satisfactory accuracy.  However, for complex real-world problems, it is often hard to specify a function class in regression that is efficiently learnable with limited data while at the same time has a small approximation error. Furthermore, it is in general impossible to estimate the approximation error of a function class, resulting in a bias that cannot be easily quantified.  The second class of approaches are based on the idea of importance sampling (IS), which corrects the mismatch between the distributions induced by the target policy and by the behavior policy~\cite{Precup00Eligibility}.  Such approaches have the salient properties of being unbiased and independent of the size of the problem's state space, but its variance can be too large for the method to be useful when the horizon is long~\cite{mandel2014offline}.

In this work, we propose a new off-policy value evaluation estimator that can achieve the best of regression based approaches (low variance) and importance sampling based approaches (no bias).  Our contributions are three-fold:
\vspace*{-.5em}
\begin{compactenum}
\item{A simple doubly robust (DR) estimator is proposed for RL that extends and subsumes a previous off-policy estimator for contextual bandits.}
\item{The estimator's statistical properties are analyzed, which suggests its superiority over previous approaches. Furthermore, in certain scenarios, we prove that the estimator's variance matches the Cramer-Rao lower bound for off-policy value evaluation.}
\item{On benchmark problems, the new estimator is much more accurate than importance sampling baselines, while remaining unbiased in contrast to regression-based approaches.  As an application, we show how such a better estimator can benefit \emph{safe} policy iteration with a more effective policy improvement step.}
\end{compactenum}

\section{Related Work}

This paper focuses on off-policy value evaluation in finite-horizon problems, which are often a natural way to model real-world problems like dialogue systems.  The goal is to estimate the expected return of start states drawn randomly from a distribution. This differs from (and is somewhat easier than) the setting considered in some previous work, often known as off-policy \emph{policy} evaluation, which aims to estimate the whole value function~\cite{Precup00Eligibility,Precup01Off,Sutton15Emphantic}.  Both settings find important yet different uses in practice, and share the same core difficulty of dealing with distribution mismatch. 


The DR technique was first studied in statistics~\cite{Rotnitzky95Semiparametric} to improve  the robustness of estimation against model misspecification, and a DR estimator has been developed for dynamic treatment regime~\cite{Murphy01Marginal}. 
DR was later applied to policy learning in contextual bandits~\cite{langford2011doubly}, and its finite-time variance is shown to be typically lower than IS.  The DR estimator in this work extends the work of \citet{langford2011doubly} to sequential decision-making problems. In addition, we show that in certain scenarios the variance of DR 
matches the statistical lower bound of the estimation problem.

An important application of off-policy value evaluation is to ensure that a new policy to be deployed does not have degenerate performance in policy iteration; example algorithms for this purpose include conservative policy iteration~\cite{kakade2002approximately} and safe policy iteration~\cite{Pirotta13Safe}. 
More recently, \citet{thomas2015high} incorporate  lower confidence bounds with IS in approximate policy iteration to ensure that the computed policy 
meets a minimum value guarantee.  Our work compliments their interesting use of confidence intervals by providing DR as a drop-in replacement of IS. 
We show that after such a replacement, an agent can accept good policies more aggressively hence obtain higher reward, while maintaining the same level of safety against bad policies.

\section{Background}
\subsection{Markov Decision Processes}
An MDP is defined by $M = \langle S, A, P, R, \gamma\rangle$, where $S$ is the state space, $A$ is the action space, $P: S\times A\times S \to \mathbb{R}$ is the transition function with $P(s' | s, a)$ being the probability of seeing state $s'$ after taking action $a$ at state $s$,
$R: S\times A \to \mathbb{R}$ is the mean reward function with $R(s, a)$ being the immediate goodness of $(s,a)$, 
and $\gamma \in [0, 1]$ is the discount factor. Let $\mu$ be the initial state distribution. A (stationary) policy $\pi: S \times A \to [0, 1]$ assigns each state $s\in S$ a distribution over actions, where $a\in A$ has probability $\pi(a | s)$. The distribution of a $H$-step trajectory $\tau = (s_1, a_1, r_1, \ldots, s_H, a_H, r_H, s_{H+1})$ is specified by $\mu$ and $\pi$ as follows: $s_1 \sim \mu$, and for $t=1, \ldots, H$, $a_t \sim \pi_0(s_t)$, $s_{t+1} \sim P(\cdot|s_t, a_t)$, and $r_t$ has mean $R(s_t, a_t)$ conditioned on $(s_t, a_t)$. We refer to such a distribution as $\tau \sim (\mu, \pi)$ in short. 
{The $H$-step discounted value of $\pi$ is
\setlength{\belowdisplayskip}{0.5em} \setlength{\belowdisplayshortskip}{0.25em}
\setlength{\abovedisplayskip}{0.5em} \setlength{\abovedisplayshortskip}{0.25em}
\begin{align}
\truev \defeq \Ebig{\tau \sim (\mu, \pi)}{\sum_{t=1}^H \gamma^{t-1} r_t}.
\end{align}
When the value of $\pi$ is conditioned on $s_1=s$ (and $a_1=a$), we define it as the state (and action) value function $V^{\pi, H}(s)$ (and $Q^{\pi, H}(s, a)$). In some discounted problems the true horizon is infinite, but for the purpose of policy value evaluation we can still use a finite $H$ (usually set to be on the order of $O(1/(1-\gamma))$) so that $\truev$ approximates $v^{\pi, \infty}$ with a bounded error that diminishes as $H$ increases.}

\subsection{Off-policy Value Evaluation}
\label{sec:problem}

For simplicity, we assume that the data (a set of length-$H$ trajectories) is sampled using a fixed stochastic policy\footnote{Analyses in this paper can be easily extended to handle data trajectories that are associated with different behavior policies.} $\pi_0$, known as the \emph{behavior policy}. Our goal is to estimate $v^{\pi_1, H}$, the value of a given \emph{target policy} $\pi_1$ from data trajectories. Below we review two popular families of estimators for off-policy value evaluation.

\textbf{Notation}~
Since we are only interested in the value of $\pi_1$, the dependence of value functions on policy is omitted. 
In terms like $V^{\pi_1, H-t+1}(s_t)$, we also omit the dependence on horizon and abbreviate as $V(s_t)$, assuming there are $H+1-t$ remaining steps. Also, all (conditional) expectations are taken with respect to the distribution induced by $(\mu,\pi_0)$, unless stated otherwise.  Finally, 
we use the shorthand: $\Expe{t}{\cdot} \defeq$ $\Expe{}{\cdot \cond s_1, a_1, \ldots, s_{t-1}, a_{t-1}}$ for conditional expectations, and $\var{t}{\cdot}$ for variances similarly.

\subsubsection{Regression Estimators}
\label{sec:mdl-estm}

If the true parameters of the MDP are known, the value of the target policy can be computed recursively by the Bellman equations: let $V^0(s)\equiv0$, and for $h=1, 2, \ldots, H$,
\begin{align}
Q^{h}(s, a) \defeq &~ \Expe{s' \sim P(\cdot | s, a)}{ R(s, a) + \gamma V^{h-1}(s')}\,, \\
V^{h}(s) \defeq &~ \Expe{a\sim \pi_1(\cdot|s)}{Q^{h}(s, a)}\,.
\label{eq:bellman}
\end{align}
This suggests a two-step, \emph{regression based} procedure for off-policy value evaluation: first, fit an MDP model $\M$ from data; second, compute the value function from \eqnref{eq:bellman} using the estimated parameters $\P$ and $\R$.  Evaluating the resulting value function, $\V^{H}(s)$, on a sample of initial states and the average will be an estimate of $v^{\pi_1,H}$. (Alternatively, one could generate artificial trajectories for evaluation without explicitly referring to a model~\citep{Fonteneau13Batch}.) When an exact state representation is used and each state-action pair appears sufficiently often in the data, such regression estimators have provably low variances and negligible biases~\cite{mannor2007bias}, and often outperform alternatives in practice~\cite{paduraru2013off}. 

However, real-world problems usually have a large or even infinite state space, and many state-action pairs will not be observed even once in the data, rendering the necessity of generalization in model fitting. To generalize, 
one can either apply function approximation to fitting $\M$ \cite{jong2007model,grunewalder2012modelling}, or to fitting the value function directly~\cite{bertsekas1996neuro,Sutton98Reinforcement,dann2014policy}. While the use of function approximation makes the problem tractable, it can introduce bias to the estimated value when the MDP parameters or the value function cannot be represented in the corresponding function class. Such a bias is in general hard to quantify from data, thus breaks the credibility of estimations given by regression based approaches~\cite{farahmand2011model,Marivate15Improved,jiang2015abstraction}.

\subsubsection{Importance Sampling Estimators}

The IS estimator provides an unbiased estimate of $\pi_1$'s value by averaging the following function of each trajectory $(s_1\I, a_1\I, r_1\I, \ldots, s_{H+1}\I)$ in the data: define the per-step importance ratio as $\rho_t\I \defeq \pi_1(a_t\I | s_t\I) / \pi_0(a_t\I | s_t\I)$, and the cumulative importance ratio $\rho_{1:t}\defeq\prod_{t'=1}^t \rho_{t'}$; the basic (trajectory-wise) IS estimator, and an improved step-wise version are given as follows:
\begin{align}
\textstyle \is\I \defeq \rho_{1:H}\I \cdot \big(\sum_{t=1}^H \gamma^{t-1} r_t\I\big), \label{eq:is}\\
\textstyle \sis\I \defeq \sum_{t=1}^H \gamma^{t-1} \rho_{1:t}\I~ r_t\I.
\label{eq:step-is}
\end{align}
Given a dataset $D$, the IS estimator is simply the average estimate over the trajectories, namely $\frac{1}{|D|}\sum_{i=1}V_{\text{IS}}^{(i)}$, where $|D|$ is the number of trajectories in $D$ and $V_{\text{IS}}^{(i)}$ is IS applied to the $i$-th trajectory. (This averaging step will be omitted for the other estimators in the rest of this paper, and we will only specify the estimate for a single trajectory). Typically, IS (even the step-wise version) suffers from very high variance, which easily grows exponentially in horizon.

A variant of IS, weighted importance sampling (WIS), is a biased but consistent estimator, given as follows together with its step-wise version: define $w_t = \sum_{i=1}^{|D|} \rho_{1:t}^{(i)} / |D|$ as the average cumulative important ratio at horizon $t$ in a dataset $D$, then from each trajectory in $D$, the estimates given by trajectory-wise and step-wise WIS are respectively
\begin{align} 
\textstyle V_{\text{WIS}} = \frac{\rho_{1:H}}{w_H}\big(\sum_{t=1}^H \gamma^{t-1} r_t\big), \label{eq:wis}\\
\textstyle V_{\text{step-WIS}} = \sum_{t=1}^H \gamma^{t-1} \frac{\rho_{1:t}}{w_t} r_t\,.
\label{eq:step-wis}
\end{align}
WIS has lower variance than IS, and its step-wise version is considered as the most practical point estimator in the IS family~\cite{precup2000temporal,thomas2015safe}. We will compare to the step-wise IS/WIS baselines in the experiments.

\subsection{Doubly Robust Estimator for Contextual Bandits}

Contextual bandits may be considered as MDPs with horizon $1$, and the sample trajectories take the form of $(s,a,r)$. 
{Suppose now we are given an estimated reward function $\R$, possibly from performing regression over a separate dataset, then the doubly robust estimator for contextual bandits~\cite{langford2011doubly} is defined as:
\setlength{\belowdisplayskip}{0.5em} \setlength{\belowdisplayshortskip}{0.25em}
\setlength{\abovedisplayskip}{0.5em} \setlength{\abovedisplayshortskip}{0.25em}
\begin{align}
\dr\I \defeq \V(s\I)  + \rho\I \left(r\I - \R(s\I, a\I)\right),
\label{eq:dr}
\end{align}
where $\rho \defeq \frac{\pi_1(a|s)}{\pi_0(a|s)}$ and $\V(s) \defeq 
\sum_a \pi_1(a|s)\R(s,a)$.
It is easy to verify that $\V(s) = \Expe{a \sim \pi_0}{\rho \R(s,a)}$, as long as $\R$ and $\rho$ are independent, which implies the unbiasedness of the estimator.  Furthermore, if $\R(s,a)$ is a good estimate of $r$, the magnitude of $r-\R(s,a)$ can be much smaller than that of $r$.  Consequently, the variance of $\rho(r-\R(s,a))$ \emph{tends to} be smaller than that of $\rho r$, implying that DR often has a lower variance than IS~\cite{langford2011doubly}.

In the case where the importance ratio $\rho$ is unknown, DR estimates both $\rho$ and the reward function from data using some parametric function classes.  The name ``doubly robust'' refers to fact that if \emph{either} function class is properly specified, the DR estimator is asymptotically unbiased, offering two chances to ensure consistency.
In this paper, however, we are only interested in DR's variance-reduction benefit. 

\textbf{Requirement of independence}~
In practice, the target policy $\pi_1$ is often computed from data, and for DR to stay unbiased, $\pi_1$ should not depend on the samples used in \eqnref{eq:dr}; the same requirement applies to IS. While $\R$ should be independent of such samples as well, it is not required that $\pi_1$ and $\R$ be independent of each other. For example, we can use the same dataset to compute $\pi_1$ and $\R$, although an independent dataset is still needed to run the DR estimator in \eqnref{eq:dr}. In other situations where $\pi_1$ is given directly, to apply DR we can randomly split the data into two parts, one for fitting $\R$ and the other for applying \eqnref{eq:dr}.
The same requirements and procedures apply to the sequential case (discussed below). In Section~\ref{sec:exp}, we will empirically validate our extension of DR in both kinds of situations.

\section{DR Estimator for the Sequential Setting}

\subsection{The Estimator}
We now extend the DR estimator for bandits to the sequential case.  A key observation is that \eqnref{eq:step-is} can be written in a recursive form. Define $\sis\ii{0} \defeq 0$, and for $t=1, \ldots, H$,
\begin{align}
\sis\ii{H+1-t} \defeq \rho_t\I \left(r_t\I + \gamma \sis\ii{H-t} \right).
\label{eq:step-is-recursive}
\end{align}
It can be shown that $\sis\ii{H}$ is equivalent to $\sis\I$ given in \eqnref{eq:step-is}. While the rewriting is straight-forward, the recursive form provides a novel and interesting insight that is key to the extension of the DR estimator: that is, we can view the step-wise importance sampling estimator as dealing with a bandit problem at each horizon $t=1,\ldots, H$, where $s_t$ is the context, $a_t$ is the action taken, and the observed stochastic return is $r_t + \gamma \sis\ii{H-t}$, whose expected value is $Q(s_t, a_t)$. 
{Then, if we are supplied with $\Q$, an estimate of $Q$ (possibly via regression on a separate dataset), we can apply the bandit DR estimator at each horizon, and obtain the following unbiased estimator: define $\dr\ii{0}\defeq0$, and
\setlength{\belowdisplayskip}{0.5em} \setlength{\belowdisplayshortskip}{0.25em}
\setlength{\abovedisplayskip}{0.5em} \setlength{\abovedisplayshortskip}{0.25em}
\begin{align}
\hspace*{-.5em}\dr\ii{H+1-t} \defeq \V(s_t\I) + \rho_t\I \Big( r_t + \gamma \dr\ii{H-t} - \Q(s_t\I, a_t\I) \Big)\,.
\label{eq:rdr}
\end{align}
The DR estimate of the policy value is then $\dr\defeq\dr\ii{H}$.}


\subsection{Variance Analysis}
\label{sec:variance}

In this section, we analyze the variance of DR in \thmref{thm:variance} and show that DR is preferable than step-wise IS when a good value function $\Q$ is available. 
The analysis is given in the form of the variance of the estimate for a \emph{single} trajectory, and the variance of the estimate averaged over a dataset $D$ will be that divided by $|D|$  due to the \iid~nature of $D$.
Due to space limit, the proof is deferred to Appendix~\ref{sec:var}.
\lihong{At some point, we may want to add a sentence or two about why we do not give a theoretical comparison to WIS.}
\begin{theorem}
\label{thm:variance}
\setlength{\belowdisplayskip}{0.5em} \setlength{\belowdisplayshortskip}{0.25em}
\setlength{\abovedisplayskip}{0.5em} \setlength{\abovedisplayshortskip}{0.25em}
$\dr$ is an unbiased estimator of $v^{\pi_1, H}$, whose variance is given recursively as follows: $\forall t=1,\ldots, H,$
\begin{align}
&~ \var{t}{\dr^{H+1-t} \csa} = \var{t}{V(s_t)}  + \Ebig{t}{\var{t}{\rho_t \Delta(s_t, a_t) \cond s_t}} \nonumber\\
& \quad + \Ebig{t}{\rho_t^2 ~\var{t+1}{r_t}} + \Ebig{t}{\gamma^2 \rho_t^2 ~\var{t+1}{\dr^{H-t}}},
\label{eq:drVar}
\end{align}
where $\Delta(s_t, a_t)\defeq\Q(s_t, a_t) - Q(s_t, a_t)$ and \\$\var{H+1}{\dr^{0} \cond s_H, a_H} = 0$.
\end{theorem}
On the RHS of \eqnref{eq:drVar}, the first 3 terms are variances due to different sources of randomness at time step $t$: state transition randomness, action stochasticity in $\pi_0$, and reward randomness, respectively; the 4th term contains the variance from future steps.  The key conclusion is that DR's variance depends on $\Q$ via the error function $\Delta = \Q - Q$ in the 2nd term, hence DR with a good $\Q$ will enjoy reduced variance, and in general outperform step-wise IS as the latter is simply DR's special case  with a trivial value function $\Q\equiv 0$.

\subsection{Confidence Intervals}
As mentioned in the introduction, an important motivation for off-policy value evaluation is to guarantee safety before deploying a policy. For this purpose, we have to characterize the uncertainty in our estimates, usually in terms of a confidence interval (CI). The calculation of CIs for DR is straight-forward, since DR is an unbiased estimator applied to \iid\ trajectories and standard concentration bounds apply. For example, Hoeffding's inequality states that for random variables with bounded range $b$, the deviation of the average from $n$ independent samples from the expected value is at most $b\sqrt{\frac{1}{2n}\log{\frac{2}{\delta}}}$ with probability at least $1-\delta$. In the case of DR, $n=\setcard{D}$ is the number of trajectories, $\delta$ the chosen confidence level, and $b$ the range of the estimate, which is a function of the maximal magnitudes of $r_t$, $\Q(s_t, a_t)$, $\rho_t$ and $\gamma$. The application of more sophisticated bounds for off-policy value evaluation in RL can be found in \citet{thomas2015high}. In practice, however, strict CIs are usually too pessimistic, and 
normal approximations are used instead~\cite{thomas2015high2}.  In the experiments, we will see how DR with normally approximated CIs can lead to more effective and reliable policy improvement than IS.

\subsection{An Extension}
\label{sec:drv2}
From \thmref{thm:variance}, it is clear that DR only reduces the variance due to action stochasticity, and may suffer a large variance even with a perfect Q-value function $\Q=Q$, as long as the MDP has substantial stochasiticity in rewards and/or state transitions. 
It is, however, possible to address such a limitation. For example,  one modification of DR that further reduces the variance in state transitions is: 
\setlength{\belowdisplayskip}{0.5em} \setlength{\belowdisplayshortskip}{0.25em}
\setlength{\abovedisplayskip}{0.5em} \setlength{\abovedisplayshortskip}{0.25em}
\begin{align}
& \drv\ii{H+1-t} = \V(s_t\I) + \rho_t\I \Big( r_t + \gamma \drv\ii{H-t} \nonumber \\
& ~~ - \R(s_t\I, a_t\I) - \gamma \V(s_{t+1}\I) {\textstyle \frac{\P(s_{t+1}\I|s_t\I, a_t\I)}{P(s_{t+1}\I|s_t\I, a_t\I)}} \Big),
\label{eq:drv2}
\end{align}
where $\P$ is the transition probability of the MDP model that we use to compute $\Q$. While we can show that this estimator is unbiased and reduces the state-transition-induced variance with a good reward \& transition functions $\R$ and $\P$ (we omit proof), it is impractical as the true transition function $P$ is unknown. However, in problems where we are confident that the transition dynamics can be estimated accurately (but the reward function may be poorly estimated),
we can assume that $P(\cdot)=\P(\cdot)$, and the last term in \eqnref{eq:drv2} becomes simply $\gamma\V(s_{t+1})$. This generally reduces more variance than the original DR at the cost of introducing a small bias. The bias is bounded in Proposition~\ref{prop:drv2bias}, whose proof is deferred to \appref{sec:drv2bias}. In Section~\ref{sec:kddexp} we will demonstrate the use of such an estimator by an experiment.
\vspace{0mm}
\begin{proposition}
Define $\epsilon = \max_{s,a}\|\P(\cdot|s,a) - P(\cdot|s,a)\|_1$.  Then, the bias of DR-v2, computed by \eqnref{eq:drv2} with the approximation $\P/P \equiv 1$, is bounded by
$\epsilon V_{\max} \sum_{t=1}^H \gamma^t$, where $V_{\max}$ is a bound on the magnitude of $\V$.
\label{prop:drv2bias}
\end{proposition} 
\vspace{-2mm}

\section{Hardness of Off-policy Value Evaluation}
\newcommand{\Mclass}{\mathcal{M}}
In Section~\ref{sec:drv2}, we showed the possibility of reducing variance due to state transition stochasticity in a special scenario. 
A natural question is whether there exists an estimator that can reduce such variance without relying on strong assumptions like $\P\approx P$. In this section, we answer this question by providing hardness results on off-policy value evaluation via the \emph{Cramer-Rao lower bound} (or \emph{C-R bound} for short), and comparing the C-R bound to the variance of DR.

Before stating the results, we emphasize that, as in other estimation problems, the C-R bound depends crucially on how the MDP is parameterized, because the parameterization captures our prior knowledge about the problem.  In general, the more structural knowledge is encoded in parameterization, the easier it is to recover the true parameters from data, and the lower the C-R bound will be. While strong assumptions (\eg, parametric form of value function) are often made in the \emph{training} phase to make RL problems tractable, 
one may not want to count them as prior knowledge in \emph{evaluation}, as every assumption made in evaluation 
decreases the credibility of the value estimate. (This is why regression-based methods are not trustworthy; see \secref{sec:mdl-estm}.) Therefore, we first present the result for the hardest case when no assumptions (other than discrete decisions \& outcomes) -- especially the Markov assumption that the last observation is a state
-- are made, to ensure the most credible estimate. A relaxed case is discussed afterwards.

\nan{Assumptions made in training phase can go into $\Q$}
\begin{definition}
An MDP is a \emph{discrete tree MDP} if \vspace*{-.5em}
\begin{itemize}[nosep,leftmargin=1em,labelwidth=*,align=left]
\item State is represented by history: that is, $s_t=h_t$, where $h_t\defeq o_1a_1 \cdots o_{t-1}a_{t-1}o_t$.  The $o_t$'s are called \emph{observations}.  We assume discrete observations and actions.
\item Initial states take the form of $s=o_1$. Upon taking action $a$,  a state $s=h$ can only transition to a next state in the form of $s' = h a o$, with probability $P(o|h,a)$. 
\item As a simplification, we assume $\gamma=1$, and non-zero rewards only occur at the end of each trajectory. 
An additional observation $o_{H+1}$
encodes the reward randomness so that reward function $R(h_{H+1})$ is deterministic, and an MDP is solely parameterized by transition probabilities.
\end{itemize}
\end{definition}

\begin{theorem}
\label{thm:crlb}
For discrete tree MDPs, the variance of any unbiased off-policy value estimator is lower bounded by \vspace*{-2mm}
\begin{align}
\sum_{t=1}^{H+1} \Ebig{}{\rho_{1:(t-1)}^2 \var{t}{V(s_t)}}\,.
\label{eq:crlb}
\end{align}
\end{theorem} 
\begin{observation}
\lihong{We could call it a corollary if you want to make it more prominent.}
The variance of DR applied to a discrete tree MDP when $\Q = Q$ is equal to \eqnref{eq:crlb}.
\label{thm:match}
\end{observation}
\nan{Discuss more about how to interpret this result: if we have $\Q=Q$ why do we need anything else?}
\lihong{This is a good point.  By comparing \thmref{thm:variance} and the LB, we can argue that DR improves IS towards the LB.  The proximity to LB depends on $\Delta$.}
\vspace*{-1.5em}
\begin{proof}[Proof of \obsref{thm:match}]
The result follows directly by unfolding the recursion in \eqnref{eq:drVar} and noticing that $\Delta\equiv 0$, $\var{t+1}{r_t}\equiv 0$ for $t<H$, and $\var{H+1}{V(s_{H+1})} = \var{H+1}{r_H}$.
\end{proof}
\vspace*{-1em}
\textbf{Implication}~
When minimal prior knowledge is available, the lower bound in \thmref{thm:crlb} equals the variance of DR with a perfect Q-value function, hence the part of variance due to state transition stochasticity (which DR fails to improve even with a good Q-value function) is intrinsic to the problem and cannot be eliminated without extra knowledge.  Moreover, the more accurate $\Q$ is, the lower the variance DR tends to have. A related hardness result is given by \citet{Li15TowardMinimax} for MDPs with \emph{known}  transition probabilities.

\textbf{Relaxed Case}~
In \appref{app:dag}, we discuss a relaxed case where the MDP has a Directed Acyclic Graph (DAG) structure, allowing different histories of the same length to be identified as the same state, making the problem easier than the tree case. The two cases share almost identical proofs, and below we give a concise proof of \thmref{thm:crlb}; readers can consult \appref{app:dag} for a fully expanded version. 
\vspace*{-1.2em}
\newcommand{\hm}{g(h_{H+1})}
\newcommand{\diag}{\text{diag}}
\newcommand{\normF}{\widehat{F}}
\begin{proof}[\textbf{Proof of Theorem~\ref{thm:crlb}}]
In the proof, it will be convenient to index rows and columns of a matrix (or vector) by histories, so that $A_{h,h'}$ denotes the $(h,h')$ entry of matrix $A$.  Furthermore, given a real-valued function $f$, $[f(h,h')]_{h,h'}$ denotes a  matrix  whose $(h,h')$ entry is given by $f(h,h')$.

We parameterize a discrete tree MDP by $\mu(o)$ and $P(o|h, a)$, for $h$ of length $1, \ldots, H$. For convenience, we treat $\mu(o)$ as $P(o|\emptyset)$, and the model parameters can be encoded as a vector $\theta$ with $\theta_{hao} = P(o|h, a)$, where $ha$ contains $|ha| = 0, \ldots, H$ alternating observations \& actions.

These parameters are subject to the normalization constraints that have to be taken into consideration in the C-R bound, namely $\forall h, a, \sum_{o\in O} P(o|h, a) =1$. In matrix form, we have $F\theta = \mathbf{1}$, where $F$ is a block-diagonal matrix with each block being a row vector of 1's; specifically, $F_{ha, h'a'o} = \1{ha=h'a'}$. Note that $F$ is the Jacobian of the constraints.  Let $U$ be a matrix whose column vectors consist of an orthonormal basis for the null space of $F$. {From \citet[Eqn.~(3.3) and Corollary~3.10]{moore2010theory}, we obtain a Constrained Cramer-Rao Bound (CCRB): 
\setlength{\belowdisplayskip}{0.5em} \setlength{\belowdisplayshortskip}{0.5em}
\setlength{\abovedisplayskip}{0.5em} \setlength{\abovedisplayshortskip}{0.5em}
\begin{align}
K U(U^\top I U)^{-1} U^\top K^\top,
\label{eq:ccrb}
\end{align}
where $I$ is the Fisher Information Matrix (FIM) without taking the constraints into consideration, and $K$ the Jacobian of the quantity $v^{\pi_1, H}$ that we want to estimate. Our calculation of the CCRB consists of four main steps.}

\textbf{1) Calculation of $I$:}~  
By definition, the FIM $I$ is computed as
$\Ebig{}{\left(\frac{\partial \log P_0(h_{H+1})}{\partial \theta}\right) \left(\frac{\partial \log P_0(h_{H+1})}{\partial \theta}\right)^\top}$,
with {\small $P_0(h_{H+1}) \defeq \mu(o_1)  \pi_0(a_1 | o_1) P(o_2 | o_1, a_1) \ldots P(o_{H+1}|h_H, a_H)$} being the probability of observing $h_{H+1}$ under policy $\pi_0$.

Define a new notation $\hm$ as a vector of indicator functions, such that $g(h_{H+1})_{hao} =1$ whenever $hao$ is a prefix of $h_{H+1}$. Using this notation, we have 
$
\frac{\partial \log P_0(h_{H+1})}{\partial \theta} = \theta^{\circ -1} \circ \hm,
$
where $\circ$ denotes element-wise power/multiplication. We rewrite the FIM as $I =  \Ebig{}{[\theta_{h}^{-1}\theta_{h'}^{-1}]_{h,h'} \circ \left(\hm \hm^\top\right)}
= [\theta_{h}^{-1}\theta_{h'}^{-1}]_{h,h'} \circ  \Expe{}{ \hm \hm^\top}$.
Now we compute $\Expe{}{ \hm \hm^\top}$. This matrix takes 0 in all the entries indexed by $hao$ and $h'a'o'$ when neither of the two strings is a prefix of the other. 
For the other entries, without loss of generality, assume $h'a'o'$ is a prefix of $hao$; the other case is similar as $I$ is symmetric. Since $\hm_{hao}\hm_{h'a'o'} = 1$ if and only if $hao$ is a prefix of $h_{H+1}$, we have $\Expe{}{\hm_{(hao)}\cdot\hm_{(h'a'o')}} = P_0(hao)$, and consequently $I_{(hao), (h'a'o')} = \frac{P_0(hao)}{P(o|h, a)P(o'|h',a')} = \frac{P_0(ha)}{P(o'|h',a')}$.

\textbf{2) Calculation of $(U^\top I U)^{-1}$:}~
Since $I$ is quite dense, it is hard to compute the inverse of $U^\top I U$ directly. Note, however, that for any matrix $X$ with matching dimensions,
$
U^\top I U = U^\top ( F^\top X^\top + I + XF) U\,,
$
because by definition $U$ is orthogonal to $F$. Observing this, we design $X$ to make $D = F^\top X^\top + I + XF$ diagonal so that $U^\top D U$ is easy to invert. This is achieved by letting $X_{(h'a'o'), (ha)}=0$ except when $h'a'o'$ is a prefix of $ha$, in which case we set $X_{(h'a'o'),(ha)} = -\frac{P_0(ha)}{P(o'|h',a')}$. It is not hard to verify that $D$ is diagonal with $D_{(hao), (hao)} = I_{(hao), (hao)} = \frac{P_0(ha)}{P(o|h,a)}$.

With the above trick, we have $(U^\top I U)^{-1} = (U^\top D U)^{-1}$. Since CCRB is invariant to the choice of $U$, 
we choose $U$ to be $\diag(\{U_{(ha)}\})$, where $U_{(ha)}$ is a diagonal block with columns forming an orthonormal basis of the null space of the none-zero part of $F_{(ha), (\cdot)}$ (an all-1 row vector). It is easy to verify that such $U$ exists and is column orthonormal, with $FU=[0]_{(ha),(ha)}$.
We also rewrite $D = \diag(\{D_{(ha)}\})$ where $D_{(ha)}$ is a diagonal matrix with $(D_{(ha)})_{o,o} = \frac{P_0(ha)}{P(o|h,a)}$, and we have
$
U (U^\top I U)^{-1} U^\top 
= \diag(\{U_{(ha)}\big(U_{(ha)}^\top D_{(ha)} U_{(ha)}\big)^{-1} U_{(ha)}^\top\}).$

The final step is to notice that each block in the expression above is simply $\frac{1}{P_0(ha)}$ times the CCRB of a multinomial distribution $p=P(\cdot |h, a)$, which is $\diag(p)- pp^\top$~ \citep[Eqn.~(3.12)]{moore2010theory}. 

\textbf{3) Calculation of $K$:}~ Recall that we want to estimate \vspace*{-.5em}
\begin{align*}
 v =v^{\pi_1, H} \nonumber = & ~ \textstyle \sum_{o_1} \mu(o_1) \sum_{a_1} \pi_1(a_1 | o_1) \,\, \cdots\\ 
 &~ \textstyle \sum_{o_{H+1}} P(o_{H+1}|h_H, a_H) R(h_{H+1})\,.
\end{align*}
Its Jacobian, $K= \partial v / \partial \theta$, can be computed by $K_{(hao)} = P_1(ha) V(hao)$, where $P_1(o_1a_1\cdots o_ta_t) \defeq \mu(o_1)\pi_1(a_1)\cdots P(o_t|h_{t-1},a_{t-1}) \pi_1(a_t|h_t)$ is the probability of observing a sequence under policy $\pi_1$.

\textbf{4) The C-R bound:}~ Putting all the pieces together, \eqnref{eq:ccrb} is equal to \vspace*{-.5em}
\begin{align*}
&~ \textstyle \sum_{ha} \frac{P_1(ha)^2}{P_0(ha)} \Big(\sum_{o} P(o|h,a) V(hao)^2 \\ & \qquad \textstyle - \big(\sum_{o} P(o|h, a) V(hao)\big)^2\Big) \\
= &~ \textstyle \sum_{t=0}^H \sum_{|ha|=t} P_0(ha) \frac{P_1(ha)^2}{P_0(ha)^2} \var{}{V(hao) \cond h, a}.
\end{align*}
Noticing that $P_1(ha)/P_0(ha)$ is the cumulative importance ratio, and $\sum_{|ha|=t} P_0(ha) (\cdot)$ is taking expectation over sample trajectories, the lower bound is equal to \vspace*{-0.5em}
\begin{align*}
\sum_{t=0}^H \Ebig{}{\rho_{1:t}^2 \var{t+1}{V(s_{t+1})}}
= \sum_{t=1}^{H+1} \Ebig{}{\rho_{1:(t-1)}^2 \var{t}{V(s_{t})}}.
\end{align*}
\renewcommand{\qedsymbol}{}
\end{proof}
\vspace*{-2.5em}

\section{Experiments}
\label{sec:exp}
Throughout this section, we will be concerned with the comparison among the following estimators.  For compactness, 
we drop the prefix ``step-wise'' from step-wise IS \& WIS. 
Further experiment details can be found in Appendix~\ref{app:exp}.
\begin{compactenum}
\item (IS) Step-wise IS of \eqnref{eq:step-is};
\item (WIS) Step-wise WIS of \eqnref{eq:step-wis};
\item (REG) Regression estimator (details to be specified for each domain in the ``model fitting'' paragraphs);
\item (DR) Doubly robust estimator of \eqnref{eq:rdr};
\item (DR-bsl) DR with a state-action independent $\Q$.
\end{compactenum}

\subsection{Comparison of Mean Squared Errors}
\label{sec:exp1}
In these experiments, we compare the accuracy of the point estimate given by each estimator. 
For each domain, a policy $\pi\train$ is computed as the optimal policy of the MDP model estimated from a training dataset $D\train$ (generated using $\pi_0$), and the target policy $\pi_1$ is set to be $(1 - \alpha) \pi\train + \alpha \pi_0$ for $\alpha \in\{0, 0.25, 0.5, 0.75\}$. 
The parameter $\alpha$ controls similarity between $\pi_0$ and $\pi_1$.  A larger $\alpha$ tends to make off-policy evaluation easier, at the cost of yielding a more conservative policy when $\pi\train$ is potentially of high quality.

We then apply the five estimators on a separate dataset $D\eval$ to estimate the value of $\pi_1$, compare the estimates to the groundtruth value, and take the average estimation errors across multiple draws of $D\eval$. 
Note that for the DR estimator, the supplied Q-value function $\Q$ should be independent of the data used in \eqnref{eq:rdr} to ensure unbiasedness.  We therefore split $D\eval$ further into two subsets $D\mdl$ and $D\test$, estimate $\Q$ from $D\mdl$ and apply DR on $D\test$. 

In the above procedure, DR does not make full use of data,  as the data in $D\mdl$ do not go into the sample average in \eqnref{eq:rdr}.
To address this issue, we propose a more data-efficient way of applying DR in the situation when $\Q$ has to be estimated from (a subset of) $D\eval$, and we call it \emph{$k$-fold DR}, inspired by $k$-fold cross validation in supervised learning:
we partition $D\eval$ into $k$ subsets, apply \eqnref{eq:dr} to each subset with $\Q$ estimated from the remaining data, and finally average the estimate over all subsets. Since the estimate from each subset is unbiased, the overall average remains unbiased, and has lower variance since all trajectories go into the sample average. 
We only show the results of $2$-fold DR as model fitting is time-consuming.


\begin{figure}[t]
\centering
\includegraphics[width=.95\columnwidth]
{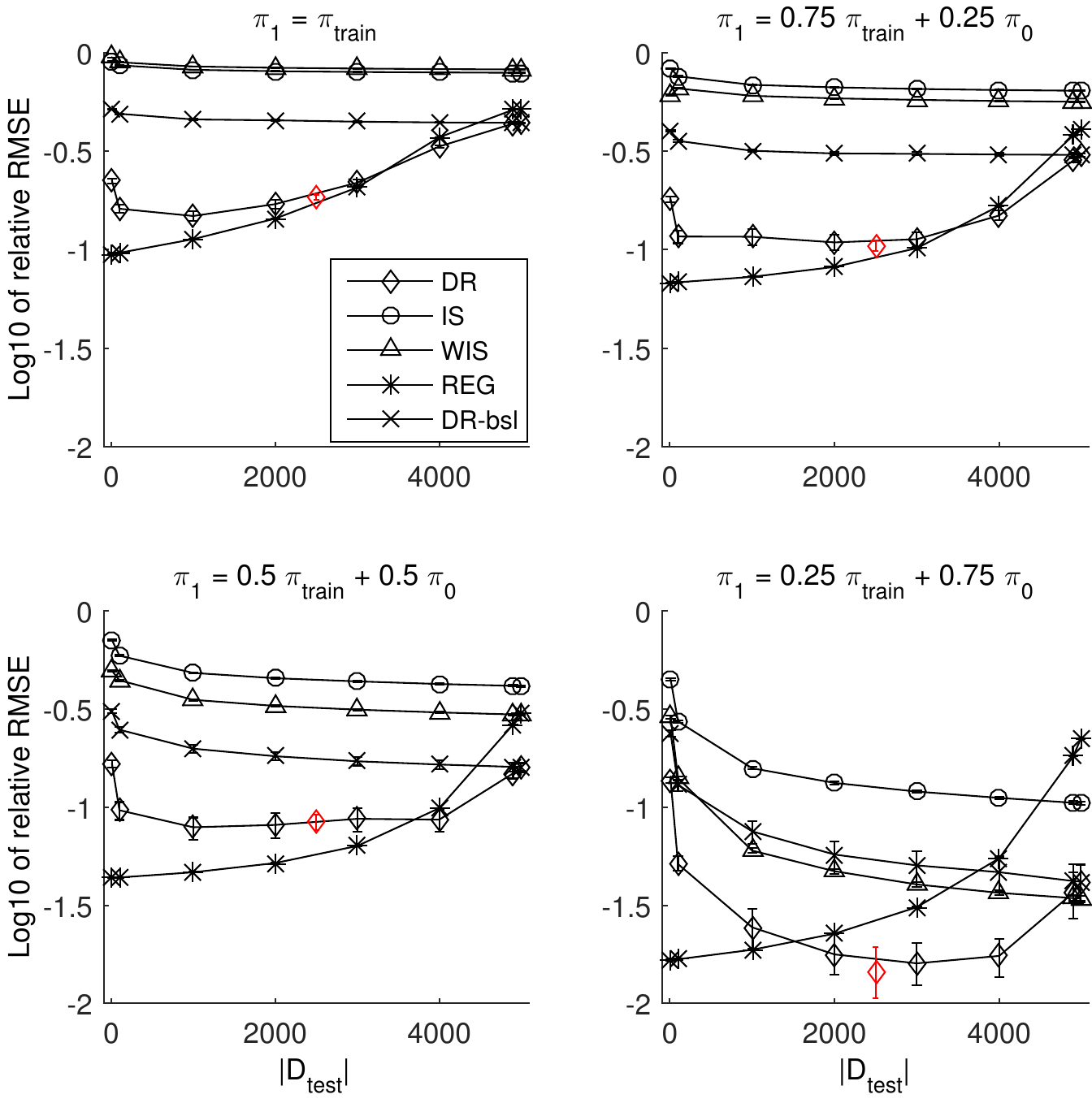}
\caption{Comparison of the methods as point estimators on Mountain Car. $5000$ trajectories are generated for off-policy evaluation, and all the results are from over $4000$ runs. The subgraphs correspond to the target policies produced by mixing $\pi\train$ and $\pi_0$ with different portions. X-axes show the size of $D\test$, the part of the data used for IS/WIS/DR/DR-bsl. The remaining data are used by the regression estimator (REG; DR uses it as $\Q$). Y-axes show the RMSE of the estimates divided by the true value in logarithmic scale. We also show the error of $2$-fold DR as an isolated point (${\color{red}\diamond}$).}
\label{fig:mc}
\end{figure}

\subsubsection{Mountain Car}
\label{sec:mcspec}

\textbf{Domain description}~
Mountain car is a popular RL benchmark problem with a $2$-dimensional continuous state space, $3$ discrete actions, and deterministic dynamics~\cite{singh1996reinforcement}. 
We set the horizon to be $100$ time steps, and the initial state distribution to uniformly random. The behavior policy is uniformly random over the $3$ actions. 

\textbf{Model fitting}~
We use state aggregations: the two state variables are multiplied by $2^6$ and $2^8$ respectively, and the rounded integers are treated as the abstract state. We then estimate an MDP model from data using a tabular approach. 

\textbf{Data sizes \& other details}~
We choose $|D\train|=2000$ and $|D\eval|=5000$. 
DR-bsl uses the step-dependent constant $\Q(s_t, a_t) = R_{\min}(1 - \gamma^{H-t+1})/(1-\gamma)$.

\textbf{Results}~
See \figref{fig:mc} for the errors of IS/WIS/DR-bsl/DR on $D\test$, and REG on $D\mdl$. As $|D\test|$ increases, IS/WIS gets increasingly better, while REG gets worse as $D\mdl$ contains less data. Since DR depends on both halves of the data, it achieves the best error at some intermediate $|D\test|$, and beats using all the data for IS/WIS in all the $4$ graphs. DR-bsl shows the accuracy of DR with $\Q$ being a constant guess, and it already outperforms IS/WIS most of the time.

\begin{figure}[t]
\centering
\includegraphics[width=.95\columnwidth] 
{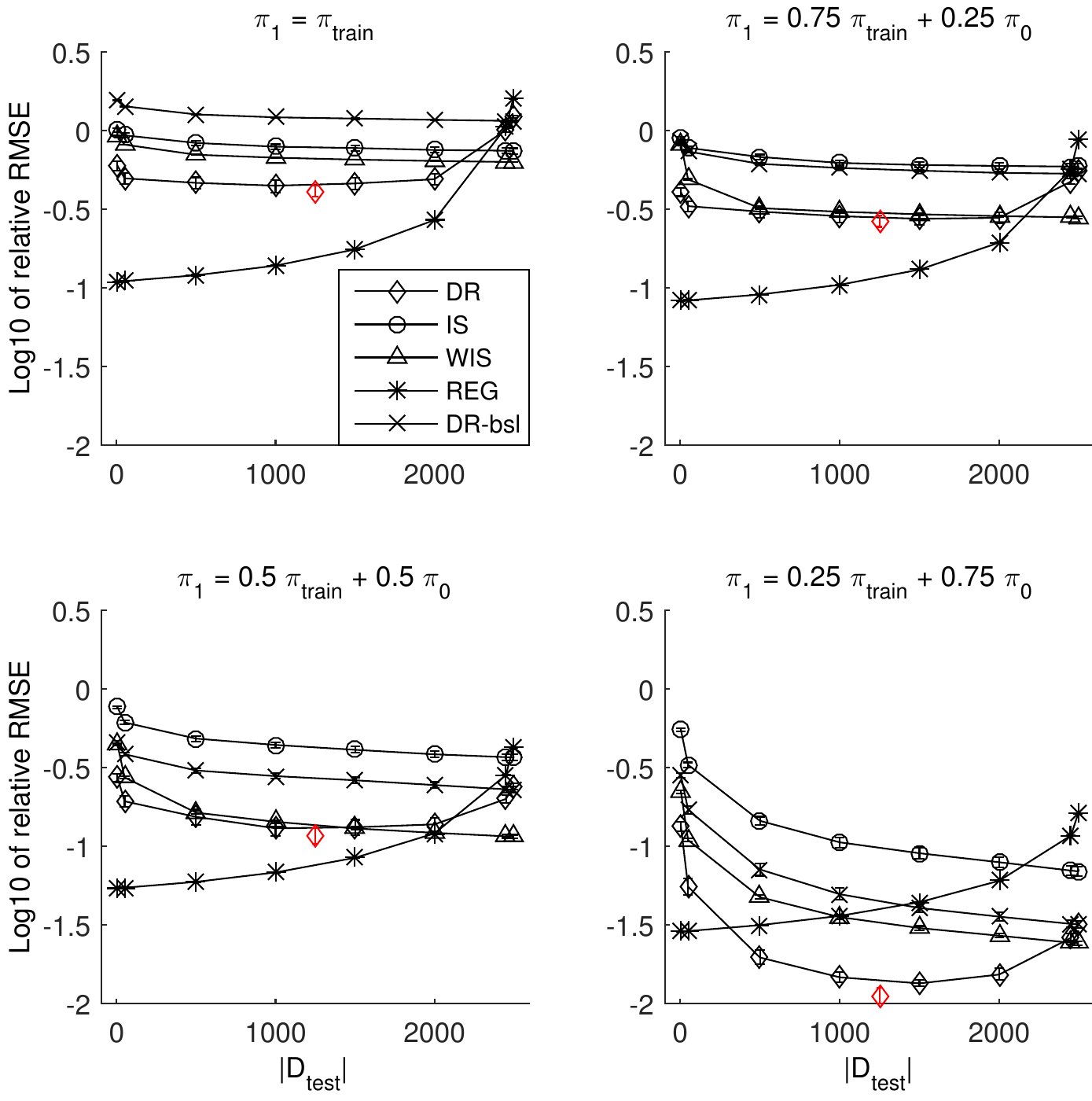}
\caption{Comparison of the methods as point estimators on Sailing ($4000$ runs). $2500$ trajectories are used in off-policy evaluation.} 
\label{fig:sail}
\end{figure}

\subsubsection{Sailing}

\textbf{Domain description}~
The sailing domain~\cite{kocsis2006bandit} is a stochastic shortest-path problem, where the agent sails on a grid. 
A state contains $4$ integer variables, each representing location or direction. There are 8 actions, and $R_{\min} = -3-4\sqrt{2}$ and $R_{\max}=0$. 

\textbf{Model fitting}~
We apply Kernel-based Reinforcement Learning~\cite{ormoneit2002kernel} with smoothing kernel $\exp(-\|\cdot\| / 0.25)$, where $\|\cdot\|$ is the $\ell_2$-distance in $S\times A$. 

\textbf{Data sizes \& other details}~
The data sizes are $|D\train|=1000$ and $|D\eval|=2500$.
DR-bsl uses the step-dependent constant $\Q(s_t, a_t) = R_{\min}/2 \cdot (1 - \gamma^{H-t+1})/(1-\gamma)$.

\textbf{Results}~
See \figref{fig:sail}. The results are qualitatively similar to Mountain Car results in \figref{fig:mc}, except that: (1) WIS is as good as DR in the 2nd and 3rd graph; (2) in the 4th graph, DR with a $3$:$2$ split outperforms all the other estimators (including the regression estimator) with a significant margin,  and a further improvement is achieved by $2$-fold DR. 

\begin{figure}[t]
\centering
\includegraphics[width=0.8\columnwidth] 
{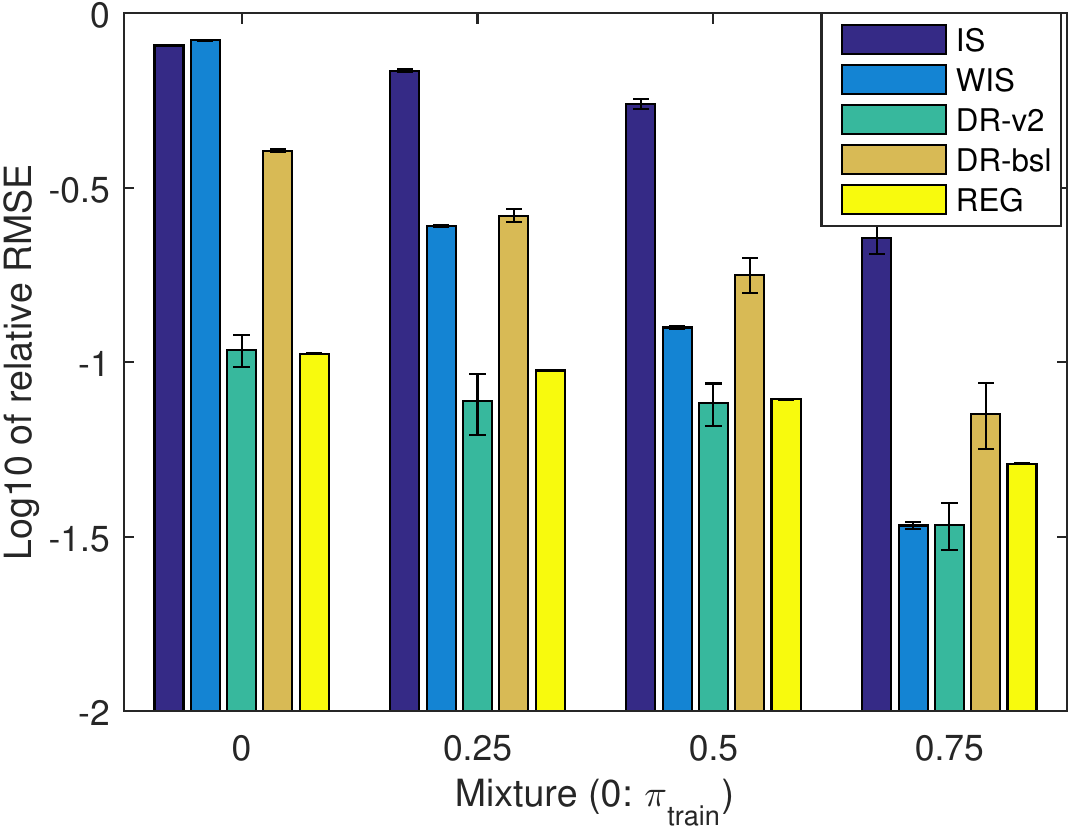}
\vspace{-1mm}
\caption{Results on the donation dataset, averaged over about $5000$ runs. 
DR-v2 is the estimator given in \eqnref{eq:drv2} with the $2$-fold trick. The whole dataset is applied to other estimators. X-axis shows the portion of which $\pi_0$ is mixed into $\pi\train$.} 
\label{fig:donation}
\vspace{-5mm}
\end{figure}

\subsubsection{KDD Cup 1998 Donation Dataset}
\label{sec:kddexp}
In the last domain, we use the donation dataset from KDD Cup 1998~\cite{kddcup1998}, which records the email interactions between an agent and potential donators. A state contains 5 integer features, and there are $12$ discrete actions. All trajectories are $22$-steps long  
and there is no discount. Since no groundtruth values are available for the target policies, we fit a simulator from the true data (see appendix for details), and use it as groundtruth for everything henceforward: the true value of a target policy is computed by Monte-Carlo roll-outs in the simulator, and the off-policy evaluation methods also use data generated from the simulator (under a uniformly random policy). Among the compared estimators, we replace DR with DR-v2 (\secref{sec:drv2}; reason explained below), 
and use the $2$-fold trick. 

The MDP model used to compute $\Q$ is estimated as follows: each state variable is assumed to evolve independently (a reasonable assumption for this dataset), and the marginal transition probabilities are estimated using a tabular approach, which is exactly how the simulator is fit from real data. Reward function, on the other hand, is fit by linear regression 
using the first $3$ state features (on the contrast, all the $5$ features are used when fitting the simulator's reward function). 
Consequently, we get a model with an almost perfect transition function and a relatively inaccurate reward function, and DR-v2 is supposed to work well in such a situation. The results are shown in \figref{fig:donation}, where DR-v2 is the best estimator in all situations: it beats WIS when $\pi_1$ is far from $\pi_0$, and beats REG when $\pi_1$ and $\pi_0$ are close.

\subsection{Application to Safe Policy Improvement}

In this experiment, we apply the off-policy value evaluation methods in safe policy improvement. Given a batch dataset $D$, the agent uses part of it ($D\train$) to find candidate policies, which may be poor due to data insufficiency and/or inaccurate approximation.  
The agent then evaluates these candidates on the remaining data ($D\test$) and chooses a policy based on the evaluation. In this common scenario, DR has an extra advantage:
$D\train$ can be reused to estimate $\Q$, and it is not necessary to hold out part of $D\test$ for regression.

Due to the high variance of IS and its variants, acting greedily w.r.t.~the point estimate is not enough to promote safety. 
A typical approach is to select the policy that has the highest lower confidence bound~\cite{thomas2015high2}, 
and hold on to the current behavior policy if none of the bounds is better than the behavior policy's value. More specifically, the bound is $V_{\dagger} - C\sigma_{\dagger}$, where $V$ is the point estimate, $\sigma$ is the empirical standard error, and $C \ge 0$ controls confidence level. $\dagger$ is a placeholder for any method that works by averaging a function of sample trajectories; examples considered in this paper are the IS and  the DR estimators.

\begin{figure}[t]
\centering
\includegraphics[width=\columnwidth] 
{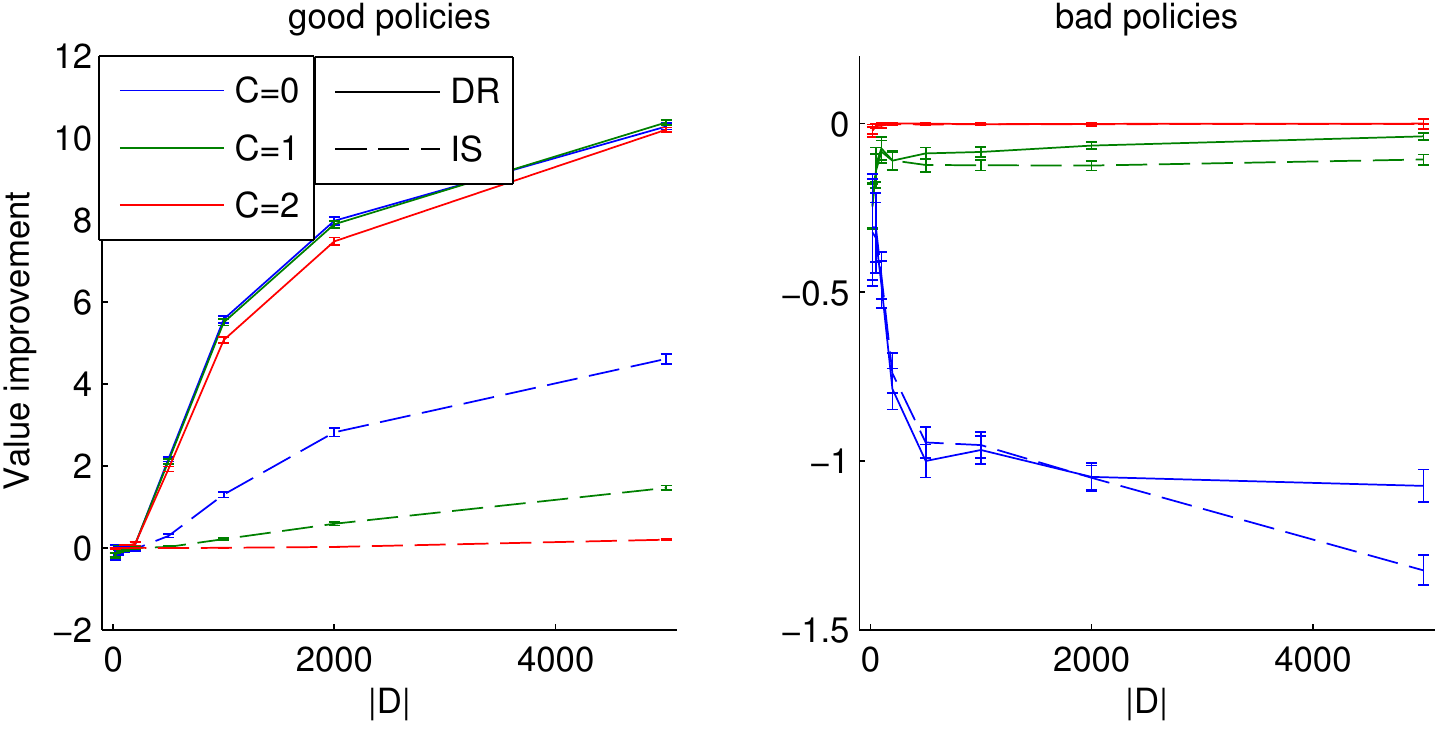}
\vspace{-3mm}
\caption{Safe policy improvement in Mountain Car. X-axis shows the size of data and y-axis shows the true value of the recommended policy subtracted by the value of the behavior policy.}
\label{fig:imprv}
\vspace{-1mm}
\end{figure}

The experiment is conducted in Mountain Car, and most of the setting is the same as Section~\ref{sec:mcspec}. Since we do not address the exploration-exploitation problem, we keep the behavior policy fixed as uniformly random, and evaluate the recommended policy once in a while as the agent gets more and more data. 
The candidate policies are generated as follows: we split $|D|$ so that $|D\train|/|D| \in \{0.2,0.4,0.6,0.8\}$; 
for each split, we compute optimal $\pi\train$ from the model estimated on $D\train$, mix $\pi\train$ and $\pi_0$ with rate $\alpha\in\{0, 0.1, \ldots, 0.9\}$, compute their confidence bounds by applying IS/DR on $D\setminus D\train$, and finally pick the policy with the highest score over all splits and $\alpha$'s.

The results are shown on the left panel of \figref{fig:imprv}. From the figure, it is clear that DR's value improvement largely outperforms IS, primarily because IS is not able to accept a target policy that is too different from $\pi_0$. However, here $\pi\train$ is mostly a good policy (except when $|D|$ is very small), hence the more aggressive an algorithm is, the more value it gets. As evidence, both algorithms achieve the best value with $C=0$, raising the concern that DR might make unsafe recommendations when $\pi\train$ is poor.

To falsify this hypothesis, we conduct another experiment in parallel, where we have $\pi\train$ minimize the value instead of maximizing it, resulting in policies worse than the behavior policy, and the results are shown on the right panel. Clearly, as $C$ becomes smaller, the algorithms become less safe, and with the same $C$ DR is as safe as IS if not better at $|D|=5000$. Overall, we conclude that DR can be a drop-in replacement for IS in safe policy improvement.

\section{Conclusions}
We proposed a doubly robust (DR) estimator for off-policy value evaluation, 
and showed its high accuracy as a point estimator and usefulness in safe policy improvement. Hardness results on the problem are also provided, and the variance of DR can match the lower bound in certain scenarios. 
Building on a preliminary version of this work, \citet{thomas2016data} showed that our DR estimator can be viewed as an application of control variates for variance reduction, and designed more advanced DR-based estimators.  
Future work includes applying such DR techniques to real-world problems to access its effectiveness in practice.

\clearpage
\section*{Acknowledgements}
We thank Xiujun Li for his help on the KDD donation dataset experiment. We also thank Satinder Singh, Michael Littman, Susan Murphy, Emma Brunskill, and the anonymous reviewers for their detailed and helpful comments, as well as Philip Thomas for an independent implementation and verification of the algorithm. 

Most of this work was done during Nan Jiang's internship in Microsoft Research. The work was partially supported by NSF grant IIS 1319365 for Nan Jiang in University of Michigan. Any opinions, findings, conclusions, or recommendations expressed here are those of the authors and do not necessarily reflect the views of the sponsors.

\bibliography{bib}
\bibliographystyle{icml2016}


\clearpage
\appendix

\input{appendix4}
\input{appendix3}
\input{appendix2}
\input{appendix-exp}

\end{document}

%% file: def.tex
\theoremstyle{plain}
\newtheorem{theorem}{\textbf{Theorem}}

\newtheorem{proposition}{Proposition}

\newtheorem{observation}{Observation}

\theoremstyle{definition}
\newtheorem{definition}{Definition}

\newcommand{\setcard}[1]{\left|#1\right|}

\newcommand{\Expe}[2]{\mathbb{E}_{#1}\big[#2\big]}
\newcommand{\Cond}{~\Big\vert~}
\newcommand{\cond}{~\big\vert~}

\newcommand{\I}{} 
\newcommand{\ii}[1]{^{#1}} 
\newcommand{\truev}{v^{\pi,H}}
\newcommand{\V}{\widehat{V}}
\newcommand{\Q}{\widehat{Q}}
\newcommand{\M}{\widehat{M}}
\renewcommand{\P}{\widehat{P}}
\newcommand{\R}{\widehat{R}}
\newcommand{\is}{V_{\text{IS}}}
\newcommand{\sis}{V_{\text{step-IS}}}

\newcommand{\dr}{V_{\text{DR}}}
\newcommand{\drv}{V_\text{DR-v2}}
\newcommand{\drvv}{V_\text{DR-v2'}}
\newcommand{\var}[2]{\mathbb{V}_{#1}\big[#2\big]} %

\newcommand{\1}[1]{\mathbf{1}{\{#1\}}}
\newcommand{\Ebig}[2]{\mathbb{E}_{#1}\Big[#2\Big]}
\newcommand{\eqnumbering}{\refstepcounter{equation}\tag{\theequation}}
\newcommand{\train}{_{\text{train}}}

\newcommand{\defeq}{:=}
\newcommand{\eval}{_\text{eval}}
\newcommand{\mdl}{_\text{reg}}
\newcommand{\test}{_\text{test}}\newcommand{\csa}{} 

\newcommand{\eg}{\textit{e.g.}}

\newcommand{\iid}{\textit{i.i.d.}}

\newcommand{\secref}[1]{Section~\ref{#1}}
\newcommand{\figref}[1]{Fig.\ref{#1}}

\newcommand{\eqnref}[1]{Eqn.(\ref{#1})}
\newcommand{\thmref}[1]{Theorem~\ref{#1}}
\newcommand{\obsref}[1]{Obs.\ref{#1}}
\newcommand{\appref}[1]{Appendix~\ref{#1}}

\newcommand{\lihong}[1]{\todo[color=green!20!white!80]{#1}}
\newcommand{\nan}[1]{\todo[color=blue!20!white!80]{#1}}
\renewcommand{\lihong}[1]{}
\renewcommand{\nan}[1]{}

%% file: appendix4.tex
\section{Proof of Theorem~\ref{thm:variance}}
\label{sec:var}
\begin{proof}
For the base case $t=H+1$, since $\dr^0=V(s_{H+1})=0$, it is obvious that at the $(H+1)$-th step the estimator is unbiased with 0 variance, and the theorem holds.
For the inductive step, suppose the theorem holds for step $t+1$.  At time step $t$, we have: 
\begin{align*}
&~ \var{t}{\dr^{H+1-t}} \\
= &~ \Ebig{t}{\big(\dr^{H+1-t}\big)^2} - \Big(\Expe{t}{V(s_t) }\Big)^2 \\
= &~ \mathbb{E}_t \Big[ \left(\V(s_t) + \rho_t \big( r_t + \gamma \dr^{H-t} - \Q(s_t, a_t) \big)\right)^2\\
& \quad - V(s_t)^2 \Big] + \var{t}{V(s_t)} \\
= &~\mathbb{E}_t \Big[\Big(\rho_t Q(s_t, a_t) - \rho_t \Q(s_t, a_t) + \V(s_t) \\
& \quad + \rho_t \big( r_t + \gamma \dr^{H-t} - Q(s_t, a_t) \big)\Big)^2 - V(s_t)^2 \Big] \\
& \quad + \var{t}{V(s_t)} \\
= &~\mathbb{E}_t \Big[\Big(-\rho_t \Delta(s_t, a_t) + \V(s_t) + \rho_t ( r_t - R(s_t, a_t)) \\
& \quad + \rho_t \gamma \big( \dr^{H-t} - \Expe{t+1}{V(s_{t+1})} \big)\Big)^2 \\
& \quad - V(s_t)^2 \Big] + \var{t}{V(s_t)} \eqnumbering \label{eq:proof-step0} \\
= &~ \Ebig{t}{\Ebig{t}{\big(-\rho_t \Delta(s_t, a_t) + \V(s_t)\big)^2- V(s_t)^2 \Cond s_t}} \\
& ~~ + \Ebig{t}{\Expe{t+1}{\rho_t^2 ( r_t - R(s_t, a_t))^2}} 
+ \var{t}{V(s_t)}\\
& ~~ + \Ebig{t}{\Ebig{t+1}{\Big(\rho_t \gamma \big( \dr^{H-t} - \Expe{t+1}{V(s_{t+1})} \big)\Big)^2}} \\ 
= &~ \Ebig{t}{\var{t}{-\rho_t \Delta(s_t, a_t) + \V(s_t) \cond s_t}} + \Ebig{t}{\rho_t^2~ \var{t+1}{r_t}} \\
& \quad + \Ebig{t}{\rho_t^2 \gamma^2~ \var{}{\dr^{H-t} \cond s_t, a_t}} + \var{t}{V(s_t)} \\ 
= &~ \Ebig{t}{\var{t}{\rho_t \Delta(s_t, a_t) \cond s_t}} + \Ebig{t}{\rho_t^2~ \var{t+1}{r_t}} \\
& \quad + \Ebig{t}{\rho_t^2 \gamma^2~ \var{t+1}{\dr^{H-t}}} + \var{t}{V(s_t)}. 
\end{align*}
This completes the proof. Note that from \eqnref{eq:proof-step0} to the next step, we have used the fact that conditioned on $s_t$ and $a_t$, $r_t - R(s_t, a_t)$ and $\dr^{H-t} - \Expe{t+1}{V(s_{t+1})}$ are independent and have zero means, and all the other terms are constants.  Therefore, the square of the sum equals the sum of squares in expectation. 
\end{proof}

%% file: appendix3.tex
\section{Bias of DR-v2}
\label{sec:drv2bias}
\begin{proof}[Proof of Proposition~\ref{prop:drv2bias}]
Let $\drvv$ denote \eqnref{eq:drv2} with approximation $\P=P$. Since $\drv$ is unbiased, the bias of $\drvv$ is then the expectation of $\drvv - \drv$. Define
\newcommand{\bias}{\beta}
\begin{align*}
\bias_t = \Ebig{t}{\drvv^{H+1-t} - \drv^{H+1-t}}.
\end{align*}
Then, $\beta_1$ is the bias we try to quantify, and is a constant.  In general, $\beta_t$ is a random variable that depends on $s_1, a_1, \ldots, s_{t-1}, a_{t-1}$. Now we have
\begin{align*}
\bias_t
= &~ \mathbb{E}_t \Big[ \rho_t \gamma \left( \drvv^{H-t} - \drv^{H-t} \right) \\
& \qquad - \rho_t \gamma \V(s_{t+1}) \left(\frac{\P(s_{t+1}|s_t, a_t)}{P(s_{t+1}|s_t, a_t)} -1 \right) \Big] \\
= &~ \Ebig{t}{\rho_t \gamma \bias_{t+1}} - \Ebig{t}{ \rho_t \gamma \V(s_{t+1}) \left(\frac{\P(s_{t+1}|s_t, a_t)}{P(s_{t+1}|s_t, a_t)} -1 \right)}.
\end{align*}

In the second term of the last expression, the expectation is taken over the randomness of $a_t$ and $s_{t+1}$; we keep $a_t$ as a random variable and integrate out $s_{t+1}$, and get
\begin{align*}
 &~ \Ebig{t}{ \rho_t \gamma \V(s_{t+1}) \left(\frac{\P(s_{t+1}|s_t, a_t)}{P(s_{t+1}|s_t, a_t)} -1 \right)} \\
= &~ \Ebig{t}{ \Ebig{t+1}{\rho_t \gamma \V(s_{t+1}) \left(\frac{\P(s_{t+1}|s_t, a_t)}{P(s_{t+1}|s_t, a_t)} -1 \right)}} \\
= &~ \Ebig{t}{ \rho_t \gamma \sum_{s'}P(s'|s_t, a_t) \V(s') \left(\frac{\P(s'|s_t, a_t)}{P(s'|s_t, a_t)} -1 \right)} \\
= &~ \Ebig{t}{ \rho_t \gamma \sum_{s'} \V(s') \left(\P(s'|s_t, a_t) -P(s'|s_t, a_t) \right)}.
\end{align*}
Recall that the expectation of the importance ratio is always $1$, hence
\begin{align*}
\bias_t \le &~ \Ebig{t}{\rho_t \gamma \left(\bias_{t+1} + \epsilon V_{\max}\right)} \\
= &~ \Ebig{t}{\rho_t \gamma \bias_{t+1}} + \gamma \epsilon V_{\max}.
\end{align*}
%
With an abuse of notation, we reuse $\bias_t$ as its maximal absolute magnitude over all sample paths $s_1, a_1, \ldots, s_{t-1}, a_{t-1}$. Clearly we have $\bias_{H+1} = 0$, and
\begin{align*}
\bias_t \le \gamma (\bias_{t+1} + \epsilon V_{\max}).
\end{align*}
Hence, $\beta_1 \le \epsilon V_{\max} \sum_{t=1}^H \gamma^t$.
\end{proof}

%% file: appendix2.tex
\section{Cramer-Rao bound for discrete DAG MDPs}
\label{app:dag}

Here, we prove a lower bound for the relaxed setting where the MDP is a layered Directed Acyclic Graph instead of a tree. In such MDPs, the regions of the state space reachable in different time steps are disjoint (just as tree MDPs), but trajectories that separate in early steps can reunion at a same state later. 
\begin{definition}[Discrete DAG MDP]
An MDP is a \emph{discrete Directed Acyclic Graph (DAG) MDP} if: 
\vspace*{-.5em}
\begin{compactitem}
\item The state space and the action space are finite.
\item For any $s \in S$, there exists a unique $t \in \mathbb{N}$ such that, $\max_{\pi: S\to A} P(s_t = s \cond \pi) >0$. In other words, a state only occurs at a particular time step.
\item As a simplification, we assume $\gamma=1$, and non-zero rewards only occur at the end of each $H$-step long trajectory. We use an additional state $s_{H+1}$ to encode the reward randomness so that reward function $R(s_{H+1})$ is deterministic and the domain can be solely parameterized by transition probabilities.
\end{compactitem}
\end{definition}

\begin{theorem}
For discrete DAG MDPs, the variance of any unbiased estimator is lower bounded by
\[
\sum_{t=1}^{H+1} \Ebig{}{\frac{P_1(s_{t-1},a_{t-1})^2}{P_0(s_{t-1},a_{t-1})^2} \var{t}{V(s_{t})}},
\]
where for trajectory $\tau$, \\
$
P_0(\tau) = \mu(s_1)  \pi_0(a_1 | s_1) P(s_2 | s_1, a_1) \ldots P(s_{H+1}|s_H, a_H),
$\\
and $P_0(s_t, a_t)$ is its marginal probability; $P_1(\cdot)$ is similarly defined for $\pi_1$.
\label{thm:dag}
\end{theorem}

\textbf{Remark}~ Compared to Theorem~\ref{thm:crlb}, the cumulative importance ratio $\rho_{1:t-1}$ is replaced by the state-action occupancy ratio $P_1(s_{t-1}, a_{t-1}) / P_0(s_{t-1}, a_{t-1})$ in Theorem~\ref{thm:dag}. The two ratios are equal when each state can only be reached by a unique sample path. In general, however, $\Expe{}{P_1(s_{t-1}, a_{t-1})^2 / P_0(s_{t-1}, a_{t-1})^2 \var{t}{V(s_{t})}} \le \Expe{}{\rho_{1:t-1}^2 \var{t}{V(s_{t})}}$, hence DAG MDPs are easier than tree MDPs for off-policy value evaluation.

Below we give the proof of Theorem~\ref{thm:dag}, which is almost identical to the proof of Theorem~\ref{thm:crlb}. 

\renewcommand{\hm}{g(\tau)}
\begin{proof}[Proof of Theorem~\ref{thm:dag}]
We parameterize the MDP by $\mu(s_1)$ and $P(s_{t+1}|s_t, a_t)$ for $t=1, \ldots, H$. For convenience we will treat $\mu(s_1)$ as $P(s_1|\emptyset)$, so all the parameters can be represented as $P(s_{t+1}|s_t, a_t)$ (for $t=0$ there is a single $s_0$ and $a$). These parameters are subject to the normalization constraints that have to be taken into consideration in the Cramer-Rao bound, namely $\forall t, s_t, a_t, \sum_{s_{t+1}} P(s_{t+1}|s_t, a_t) =1$.
\begin{align}
\begin{bmatrix}
1  \cdots  1 & & & \\
 &   1  \cdots  1 & &\\
 &  &   \ddots &\\
 &   & &   1 \cdots 1
\end{bmatrix} \theta = \begin{bmatrix} 1\\ 1\\ \vdots\\ 1 \end{bmatrix}
\end{align}
where $\theta_{s_t, a_t, s_{t+1}} = P(o|s_t, a_t)$. The matrix on the left is effectively the Jacobian of the constraints, which we denote as $F$. We index its rows by $(s_t, a_t)$, so $F_{(s_t, a_t), (s_t, a_t, s_{t+1})} = 1$ and other entries are $0$. Let $U$ be a matrix whose column vectors consist an orthonormal basis for the null space of $F$. From \citet[Eqn.~(3.3) and Corollary~3.10]{moore2010theory}, we have the Constrained Cramer-Rao Bound (CCRB) being\footnote{In fact, existing literature on Contrained Cramer-Rao Bound does not deal with the situation where the unconstrained parameters break the normalization constraints (which we are facing). However, this can be easily tackled by changing the model slightly to $P(o|h, a) = \theta_{hao} / \sum_{o'} \theta_{hao'}$, which resolves the issue and gives the same result.
} (the dependence on $\theta$ in all terms are omitted):
\begin{align}
K U(U^\top I U)^{-1} U^\top K^\top,
\label{eq:ccrb2}
\end{align}
where $I$ is the Fisher Information Matrix (FIM), and $K$ is the Jacobian of the quantity we want to estimate; they are computed below.  We start with $I$, which is
\begin{align}
I = \Ebig{}{\left(\frac{\partial \log P_0(\tau)}{\partial \theta}\right) \left(\frac{\partial \log P_0(\tau)}{\partial \theta}\right)^\top}.
\end{align}
To calculate $I$, we define a new notation $\hm$, which is a vector of indicator functions and $\hm_{s_t,a_t,s_{t+1}} =1$ when $(s_t, a_t, s_{t+1})$ appears in trajectory $\tau$. Using this notation, we have 
\begin{align}
\frac{\partial \log P_0(\tau)}{\partial \theta} = \theta^{\circ -1} \circ \hm,
\end{align}
where $\circ$ denotes element-wise power/multiplication. Then we can rewrite the FIM as
\begin{align}
I = &~ \Ebig{}{[\theta_{i}^{-1}\theta_{j}^{-1}]_{ij} \circ (\hm \hm^\top)} \nonumber\\
= &~ [\theta_{i}^{-1}\theta_{j}^{-1}]_{ij} \circ  \Expe{}{ (\hm \hm^\top)},
\end{align} 
where $[\theta_{i}^{-1}\theta_{j}^{-1}]_{ij}$ is a matrix expressed by its $(i, j)$-th element. Now we compute $\Expe{}{ \hm \hm^\top}$. On the diagonal, it is $P_0(s_t, a_t, s_{t+1})$, so the diagonal of $I$ is $\frac{P_0(s_t, a_t)}{P(s_{t+1}|s_t, a_t)}$; for non-diagonal entries whose row indexing and column indexing tuples are at the same time step, the value is 0; in other cases, suppose row is $(s_t, a_t, s_{t+1})$ and column is $s_{t'}, a_{t'}, s_{t'+1}$, and without loss of generality assume $t' < t$, then the entry is $P_0(s_{t'}, a_{t'}, s_{t'+1}, s_t,a_t, s_{t+1})$, with the corresponding entries in $I$ being $ \frac{P_0(s_{t'}, a_{t'}, s_{t'+1}, s_t,a_t, s_{t+1})}{P(s_{t'+1}|s_{t'}, a_{t'})P(s_{t+1}|s_{t}, a_t)} = P_0(s_{t'}, a_{t'}) P_0(s_t, a_t | s_{t'+1})$.

Then, we calculate $(U^\top I U)^{-1}$. To avoid the difficulty of taking inverse of this non-diagonal matrix, we apply the following trick to diagonalize $I$: note that for any matrix $X$ with matching dimensions,
\begin{align}
U^\top I U = U^\top ( F^\top X^\top + I + XF) U,
\end{align}
because by definition $U$ is orthogonal to $F$. We can design $X$ so that $D = F^\top X^\top + I + XF$ is a diagonal matrix, and $D_{(s_t, a_t, s_{t+1}),(s_t, a_t, s_{t+1})} = I_{(s_t, a_t, s_{t+1}),(s_t, a_t, s_{t+1})} = \frac{P_0(s_t, a_t)}{P(s_{t+1}|s_t,a_t)}$. This is achieved by having $XF$ eliminate all the non-diagonal entries of $I$ in the upper triangle without touching anything on the diagonal or below, and by symmetry $F^\top X^\top$ will deal with the lower triangle. The particular $X$ we take is $X_{(s_{t'}, a_{t'}, s_{t'+1}),(s_t, a_t)} = - P_0(s_{t'}, a_{t'}) P_0(s_t, a_t | s_{t'+1}) \mathbb{I}(t' < t)$, and it is not hard to verify that this construction diagonalizes $I$.

With the diagonalization trick, we have $(U^\top I U)^{-1} = (U^\top D U)^{-1}$. Since CCRB is invariant to the choice of $U$, and we observe that the rows of $F$ are orthogonal, 
we choose $U$ as follows: let $n_{(s_t, a_t)}$ be the number of $1$'s in $F_{(s_t, a_t), (\cdot)}$, and $U_{(s_t, a_t)}$ be the $n_{(s_t, a_t)}\times (n_{(s_t, a_t)}-1)$ matrix with orthonormal columns in the null space of $\begin{bmatrix} 1 & \ldots & 1\end{bmatrix}$ ($n_{(s_t, a_t)}$ $1$'s); finally, we choose $U$ to be a block diagonal matrix $U = \diag(\{U_{(s_t, a_t)}\})$, where $U_{(s_t, a_t)}$'s are the diagonal blocks, and it is easy to verify that $U$ is column orthonormal and $FU=0$. Similarly, we write $D = \diag(\{D_{(s_t, a_t)}\})$ where $D_{(s_t, a_t)}$ is a diagonal matrix with $(D_{(s_t, a_t)})_{s_{t+1},s_{t+1}} = P_0(s_t, a_t)/P(s_{t+1}|s_t,a_t)$, and
\begin{align*}
&~U (U^\top I U)^{-1} U^\top  = U (U^\top D U)^{-1} U^\top \\
= &~ U (\diag(\{U_{(s_t, a_t)}^\top\})\diag(\{D_{(s_t, a_t)}\}) \diag(\{U_{(s_t, a_t)}\}))^{-1} U \\
= &~ U \diag(\{\big(U_{(s_t, a_t)}^\top D_{(s_t, a_t)} U_{(s_t, a_t)}\big)^{-1}\}) U \\
= &~ \diag(\{U_{(s_t, a_t)}\big(U_{(s_t, a_t)}^\top D_{(s_t, a_t)} U_{(s_t, a_t)}\big)^{-1} U_{(s_t, a_t)}^\top\}). \eqnumbering \label{eq:diaged2}
\end{align*}
Notice that each block in \eqnref{eq:diaged2} is simply $1/P_0(s_t, a_t)$ times the CCRB of a multinomial distribution $P(\cdot |s_t, a_t)$. The CCRB of a multinomial distribution $p$ can be easily computed by an alternative formula~\citep[Eqn.~(3.12)]{moore2010theory}), which gives $\diag(p)- pp^\top$, so we have,
\begin{align}
&~ U_{(s_t, a_t)}\big(U_{(s_t, a_t)}^\top D_{(s_t, a_t)} U_{(s_t, a_t)}\big)^{-1} U_{(s_t, a_t)}^\top \nonumber \\
= &~ \frac{\diag(P(\cdot|s_t, a_t)) - P(\cdot|s_t, a_t) P(\cdot|s_t, a_t)^\top}{P_0(s_t, a_t)}.
\end{align}
We then calculate $K$. Recall that we want to estimate 
\begin{align}
 v =v^{\pi_1, H} \nonumber = & ~ \sum_{s_1} \mu(s_1) \sum_{a_1} \pi_1(a_1 | s_1) \ldots\\ 
 &~ \sum_{s_{H+1}} P(s_{H+1}|s_H, a_H) R(s_{H+1})\,,
\end{align}
and its Jacobian is $K = (\partial v / \partial \theta_t)^\top$, with $K_{(s_t, a_t, s_{t+1})} = P_1(s_t, a_t)V(s_{t+1})$, where $P_1(\tau) = \mu(s_1)\pi_1(a_1)\ldots P(s_{H+1}|s_{H},a_{H})$ and $P_1(s_t, a_t)$ is the marginal probability.

Finally, putting all the pieces together, we have \eqnref{eq:ccrb2} equal to
\begin{align*}
&~ \sum_{s_t, a} \frac{P_1(s_t, a_t)^2}{P_0(s_t, a_t)} \Big(\sum_{s_{t+1}} P(s_{t+1}|s_t,a_t) V(s_{t+1})^2 \\ & \qquad - \big(\sum_{s_{t+1}} P(s_{t+1}|s_{t}, a_t) V(s_{t+1})\big)^2\Big) \\
= &~ \sum_{t=0}^H \sum_{s_t} P_0(s_t,a_t) \frac{P_1(s_t,a_t)^2}{P_0(s_t,a_t)^2} \var{}{V(s_{t+1}) \cond s_t, a}\\
= &~ \sum_{t=0}^H \Ebig{}{\frac{P_1(s_t,a_t)^2}{P_0(s_t,a_t)^2} \var{t+1}{V(s_{t+1})}} \\
= &~ \sum_{t=1}^{H+1} \Ebig{}{\frac{P_1(s_{t-1},a_{t-1})^2}{P_0(s_{t-1},a_{t-1})^2} \var{t}{V(s_{t})}}.\qedhere
\end{align*}
\end{proof}

%% file: appendix-exp.tex
\section{Experiment Details}
\label{app:exp}

Here, we provide full details on the experiments that are omitted in the main paper due to space limit.

\subsection{Mountain Car}
\textbf{Domain Description}~~
Mountain car is a widely used benchmark problem for RL with a 2-dimensional continuous state space (position and velocity) and deterministic dynamics~\cite{singh1996reinforcement}. The state space is $[-1.2, 0.6] \times [-0.07, 0.07]$, and there are $3$ discrete actions. The agent receives $-1$ reward every time step with a discount factor $0.99$, and an episode terminates when the first dimension of state reaches the right boundary. 
The initial state distribution is set to uniformly random, and behavior policy is uniformly random over the $3$ actions. The typical horizon for this problem is $400$, which can be too large for IS and its variants, therefore we accelerate the dynamics such that given $(s, a)$, the next state $s'$ is obtained by calling the original transition function $4$ times holding $a$ fixed, and we set the horizon to $100$.
A similar modification was taken by \citet{thomas2015safe}, where every $20$ steps are compressed as one step.

\textbf{Model Construction}~~
The model we construct for this domain uses a simple discretization (state aggregation): the two state variables are multiplied by $2^6$ and $2^8$ respectively and the rounded integers are treated as the abstract state. We then estimate the model parameters from data using a tabular approach. Unseen aggregated state-action pairs are assumed to have reward $R_{\min}=-1$ and a self-loop transition. Both the models that produces $\pi\train$ and that used for off-policy evaluation are constructed in the same way.

\textbf{Data sizes \& other details}~~
The dataset sizes are $|D\train|=2000$ and $|D\eval|=5000$. 
We split $D\eval$ such that $D\test \in \{10, 100, 1000, 2000, 3000, 4000, 4900, 4990\}$. 
DR-bsl uses the step-dependent constant function
$$
\Q(s_t, a_t) = \frac{R_{\min}(1 - \gamma^{H-t+1})}{1-\gamma}.
$$
Since the estimators in the IS family typically has a highly skewed distribution, the estimates can occasionally go largely out of range, and we crop such outliers in $[V_{\min}, V_{\max}]$ to ensure that we can get statistically significant experiment results within a reasonable number of simulations. The same treatment is also applied to the experiment on Sailing.

\subsection{Sailing}

\textbf{Domain Description}~~
The sailing domain~\cite{kocsis2006bandit} is a stochastic shortest-path problem, where the agent sails on a grid (in our experiment, a map of size $10 \times 10$) with wind blowing in random directions, aiming at the terminal location on the top-right corner. The state is represented by $4$ integer variables, representing either location or direction. At each step, the agent chooses to move in one of the $8$ directions, 
(moving against the wind or running off the grid is prohibited), 
and receives a negative reward that depends on moving direction, wind direction, and other factors, ranging from $R_{\min} = -3-4\sqrt{2}$ to $R_{\max}=0$ (absorbing). The problem is non-discounting, and we use $\gamma=0.99$ for easy convergence when computing $\pi\train$.

\textbf{Model Construction}~~
We apply Kernel-based Reinforcement Learning~\cite{ormoneit2002kernel} and supply a smoothing kernel in the joint space of states and actions. The kernel we use takes the form $\exp(-\|\cdot\| / b)$, where $\|\cdot\|$ is the $\ell_2$-distance in $S\times A$,\footnote{The difference of two directions is defined as the angle between them (in degrees) divided by $45^{\circ}$. For computational efficiency, the kernel function is cropped to 0 whenever two state-action pairs deviate more than $1$ in any of the dimensions.} 
and $b$ is the kernel bandwidth, set to $0.25$.

\textbf{Data sizes \& other details}~~
The data sizes are $|D\train|=1000$ and $|D\eval|=2500$, and we split $D\eval$ such that $D\test \in \{5, 50, 500, 1000, 1500, 2000, 2450, 2495\}$. 
DR-bsl uses the step-dependent constant function
$$
\Q(s_t, a_t) = \frac{R_{\min}}{2} \frac{1 - \gamma^{H-t+1}}{1-\gamma},
$$
for the reason that in Sail $R_{\min}$ is rarely reached hence too pessimistic as a rough estimate of the magnitude of reward obtained per step.

\subsection{KDD Cup 1998 Donation Dataset}

Here are further details for experiments with the KDD donation dataset:
\begin{enumerate}
\item The size of dataset generated from the simulator for off-policy evaluation is equal to that of the true dataset (the one we use to fit the simulator at the very beginning;  there are $3754$ trajectories in that dataset).
\item The policy $\pi\train$ is generated by training a recurrent neural network on the original data to fit a Q-value function~\cite{Li15Recurrent}.
\item  Since there are many possible next-states for each state-action pair, for computational efficiency we use a sparse-sample approach when estimating $\Q$ using the fitted model $\M$: for each $(s, a)$, we randomly sample several next-states from $\P(\cdot | s, a)$, and cache them as a particle representation for the next-state distribution. The number of particles is set to $5$ which is enough to ensure high accuracy.
\end{enumerate} 